\documentclass[12pt]{article}

\usepackage{amsmath}
\usepackage{amsfonts}
\usepackage{amssymb}
\usepackage{amsthm}
\usepackage{graphicx}
\usepackage{verbatim}
\usepackage{mdwlist}
\usepackage[numbers,sort&compress]{natbib}
\makeatletter
\renewcommand\bibsection%
{
  \section*{\refname
    \@mkboth{\MakeUppercase{\refname}}{\MakeUppercase{\refname}}}
}
\makeatother
\usepackage{color}
\usepackage{xcolor,colortbl}
\usepackage{url}
\usepackage{caption}
\usepackage{subfigure}
\usepackage{enumerate}

\usepackage{algorithm}
\usepackage[noend]{algpseudocode}
\usepackage[affil-it]{authblk}

\newcommand{\lsize}{n}
\newcommand{\osize}{m}
\newcommand{\qsize}{q}
\newcommand{\odsize}{d}
\newcommand{\qdsize}{r}

\newcommand{\ispace}{\mathcal{X}}
\newcommand{\taskspace}{\mathcal{T}}
\newcommand{\objspace}{\mathcal{D}}

\newcommand{\anyspace}{\mathcal{X}}
\newcommand{\hypspace}{\mathcal{H}}

\newcommand{\kernelf}{k}
\newcommand{\gkernelf}{g}
\newcommand{\kkernelf}{\Gamma}
\newcommand{\dkernelm}{\bm{K}}
\newcommand{\tkernelm}{\bm{G}}
\newcommand{\kkernelm}{\bm{\Gamma}}

\newcommand{\regparam}{\lambda}

\newcommand{\idmatrix}{\bm{I}}

\newcommand{\trace}{\textnormal{tr}}
\newcommand{\transpose}{^\textnormal{T}}

\newcommand{\bm}[1]{\mathbf{#1}}

\newcommand{\ve}{\textnormal{vec}}

\newcommand{\predfun}{f}

\newcommand{\filterfun}{\varphi}

\DeclareMathOperator*{\argmin}{argmin}

\newtheorem{proposition}{Proposition}

\newtheorem{definition}{Definition}

\renewcommand*{\citep}{\cite}
\renewcommand*{\citet}{\cite}

\begin{document}

\title{A two-step learning approach for solving full and almost full cold start problems in dyadic prediction}
%\titlerunning{A two-step approach for cold start problems in dyadic prediction}

\author[1]{Tapio Pahikkala}
\author[2]{Michiel Stock}
\author[1]{Antti Airola}
\author[3]{Tero Aittokallio}
\author[2]{Bernard De Baets}
\author[2]{Willem Waegeman}
\affil[1]{University of Turku and Turku Centre for Computer Science, Joukahaisenkatu 3-5 B, FIN-20520, Turku, Finland, firstname.surname@utu.fi}
\affil[2]{Department of Mathematical Modelling, Statistics and Bioinformatics, Ghent University,
Coupure links 653, B-9000 Ghent, Belgium,
firstname.surname@UGent.be}
\affil[3]{Institute for Molecular Medicine Finland (FIMM), P.O. Box 20 (Tukholmankatu 8), FI-00014 University of Helsinki, Helsinki, Finland. firstname.surname@fimm.fi}

\date{}

\maketitle

\begin{abstract}
Dyadic prediction methods operate on pairs of objects (dyads), aiming to infer labels for out-of-sample dyads. We consider the full and almost full cold start problem in dyadic prediction, a setting that occurs when both objects in an out-of-sample dyad have not been observed during training, or if one of them has been observed, but very few times. A popular approach for addressing this problem is to train a model that makes predictions based on a pairwise feature representation of the dyads, or, in case of kernel methods, based on a tensor product pairwise kernel. As an alternative to such a kernel approach, we introduce a novel two-step learning algorithm that borrows ideas from the fields of pairwise learning and spectral filtering. We show theoretically that the two-step method is very closely related to the tensor product kernel approach, and experimentally that it yields a slightly better predictive performance. Moreover, unlike existing tensor product kernel methods, the two-step method allows closed-form solutions for training and parameter selection via cross-validation estimates both in the full and almost full cold start settings, making the approach much more efficient and straightforward to implement.   
\end{abstract}

\section{A subdivision of dyadic prediction methods}

Many real-world machine learning problems can be naturally represented as pairwise learning or dyadic prediction problems, for which feature representations of two different types of objects (aka a dyad) are jointly used to predict a relationship between those objects. Amongst others, applications of that kind emerge in biology (e.g. predicting protein-RNA interactions), medicine (e.g. design of personalized drugs), chemistry (e.g. prediction of binding between two types of molecules), social network analysis (e.g. link prediction) and recommender systems (e.g. personalized product recommendation). 

For many dyadic prediction problems it is extremely important to implement appropriate training and evaluation procedures. \citet{park2012flaws} make in a recent Nature-review on dyadic prediction an important distinction between four main settings. Given $\bm{t}$ and $\bm{d}$ as the feature representations of the two types of objects, those four settings can be summarized as follows:   
\begin{itemize}
\item {\bf Setting A}: Both $\bm{t}$ and $\bm{d}$ are observed during training, as parts of separate dyads, but the label of the dyad $(\bm{t},\bm{d})$ must be predicted. %\\
\item {\bf Setting B}: Only $\bm{t}$ is known during training, while $\bm{d}$ is not observed in any dyad, and the label of the dyad $(\bm{t},\bm{d})$ must be predicted.  %\\
\item {\bf Setting C}: Only $\bm{d}$ is known during training, while $\bm{t}$ is not observed in any dyad, and the label of the dyad $(\bm{t},\bm{d})$ must be predicted.  %\\
\item {\bf Setting D}: Neither $\bm{t}$ nor $\bm{d}$ occur in any training dyad, but the label of the dyad $(\bm{t},\bm{d})$ must be predicted (referred to as the full cold start problem). %\\
\end{itemize}

Setting A is of all four settings by far the most studied setting in the machine literature. Motivated by applications in collaborative filtering and link prediction, matrix factorization and related techniques are often applied to complete partially observed matrices, where missing values represent $(\bm{t},\bm{d})$ combinations that are not observed during training - see e.g. \cite{Koren2009} for a review.

Settings B and C are very similar, and a variety of machine learning methods can be applied for these settings. From a recommender systems viewpoint, those settings resemble the cold start problem (new user or new item), for which hybrid and content-based filtering techniques are often applied -- see e.g.\ \cite{Adams2010,Fang2011, menon2010loglinear, Shan2010,Zhou2012} for a not at all exhaustive list. From a bioinformatics viewpoint, Settings B and C are often analyzed using graph-based methods that take the structure of a biological network into account -- see e.g.\ \cite{Schrynemackers2013} for a recent review. When the features of $\bm{t}$ are negligible or unavailable, while those of $\bm{d}$ are informative, Setting B can be interpreted as a multi-label classification problem (binary labels), a multivariate regression problems (continuous labels) or a specific multi-task learning problem. Here as well, a large number of applicable methods exists in the literature. 

\subsection{The problem setting considered in this article}
Matrix factorization and hybrid filtering strategies are not applicable to Setting D. We will refer to this setting as the \emph{full cold start} problem, which finds important applications in domains such as bioinformatics and chemistry -- see experiments. Compared to the other three settings, Setting D has received less attention in the literature (with some exceptions, see e.g. \cite{park2009pairwise,menon2010loglinear,pahikkala2010conditional,pahikkala2013efficient}), and it will be our main focus in this article. Furthermore, we will also investigate the transition phase between Settings C and D, when $\bm{t}$ occurs very few times in the training dataset, while $\bm{d}$ of the dyad $(\bm{d},\bm{t})$ is only observed in the prediction phase. We refer to this setting as the \emph{almost full cold start} problem. 

Full and almost full cold start problems can only be solved by considering feature representations of dyads (aka side information in the recommender systems literature). Similar to several existing papers dealing with Setting D, we will consider tensor product feature representations and their kernel duals.  Such feature representations have been successfully applied in order to solve problems such as product recommendation \cite{basilico2004unifying, park2009pairwise}, prediction of protein-protein interactions \cite{Benhur2005,Kashima2009linkprob}, drug design \cite{Jacob2008}, prediction of game outcomes \cite{pahikkala2010reciprocalkm} and document retrieval \cite{pahikkala2013efficient}. For classification and regression problems a standard recipe exists of plugging pairwise kernels in support vector machines, kernel ridge regression (KRR), or any other kernel method. Efficient optimization approaches based on gradient descent \cite{park2009pairwise,Kashima2009linkprob,pahikkala2013efficient} and closed form solutions \cite{pahikkala2013efficient} have been introduced. In our theoretical and experimental analysis we will compare KRR with a tensor product pairwise kernel to the two-step approach that we introduce in this paper.  

\subsection{Formulation as a transfer learning problem}
As discussed above, dyadic prediction is closely related to several subfields of machine learning. Further on in this article we decide to adopt a multi-task learning or transfer learning terminology, using $\bm{d}$ and $\bm{t}$ to denote the feature representations of instances and tasks, respectively. From this viewpoint, Setting C corresponds to a specific instantiation of a traditional transfer learning scenario, in which the aim is to transfer knowledge obtained from already learned \emph{auxiliary tasks} to the \emph{target task} of interest \citet{pan2010surveytransfer}. Stretching the concept of transfer learning even further, in the case of so-called \emph{zero-data learning}, one arrives at Setting D, which is characterized by no available labeled training data for the target task \citep{Larochelle2008zerodata}. If the target task is unknown during the training time, the learning method must be able to generalize to it ``on the fly'' at prediction time. The only available data here is coming from auxiliary training tasks.   

We present a simple but elegant two-step approach to tackle these settings. First, a KRR model trained on auxiliary tasks is used to predict labels for the related target task. Next, a second model is constructed, using KRR on the target data, augmented by the predictions of the first phase. We show via spectral filtering that this approach is closely related to learning a pairwise model using a tensor product pairwise kernel. However, the two-step approach is much simpler to implement and it allows more heterogeneous transfer learning settings than the ordinary pairwise kernel ridge regression, as well as a more flexible model selection. Furthermore, it allows for a more efficient generalization to new tasks not known during training time, since the model built on auxiliary tasks does not need to be re-trained in such settings. In the experiments we consider three distinct dyadic prediction problems, concerning drug-target, newsgroup document similarity and protein functional similarity predictions. Our results show that the two-step transfer learning approach can be highly beneficial when there is no labeled data at all, or only a small amount of labeled data available for the target task, while in settings where there is a significant amount of labeled data available for the target task a single-task model suffices. In related work, \cite{Schrynemackers2014} have recently proposed a similar two-step approach based on tree-based ensemble methods for biological network inference.

\section{Solving full and almost full cold start problems via transfer learning}

\newcommand{\objset}{D}
\newcommand{\taskset}{T}

\newcommand{\krf}{\bm{X}}
\newcommand{\eq}[1]{\begin{align*}#1\end{align*}}
\newcommand{\eqn}[1]{\begin{align}#1\end{align}}

Adopting a multi-task learning methodology, the training set is assumed to consist of a set $\{\bm{x}_h\}_{h=1}^{\lsize}$ of object-task pairs and a vector $\bm{y}\in\mathbb{R}^\lsize$ of their real-valued labels. We assume that each training input can be represented as $\bm{x}=(\bm{d},\bm{t})$, where $\bm{d}\in\objspace$ and $\bm{t}\in\taskspace$ are the objects and tasks, respectively, and $\objspace$ and $\taskspace$ are the corresponding spaces of objects and tasks. Moreover, let $\objset=\{\bm{d}_i\}_{i=1}^{\osize}$ and $\taskset=\{\bm{t}_j\}_{j=1}^{\qsize}$ denote, respectively, the sets of distinct objects and tasks encountered in the training set with $\osize=\arrowvert\objset\arrowvert$ and $\qsize=\arrowvert\taskset\arrowvert$. We say that the training set is \emph{complete} if it contains every object-task pair with object in $\objset$ and task in $\taskset$ exactly once. For complete training sets, we introduce a further notation for the matrix of labels $\bm{Y}\in\mathbb{R}^{\osize\times\qsize}$, so that its rows are indexed by the objects in $\objset$ and the columns by the tasks in $\taskset$. In full and almost full cold start prediction problems, this matrix will not contain any target task info.

\subsection{Kernel ridge regression with tensor product kernels}
\citet{baldassarre2012multioutput} and several other authors (see \citet{alvarez2012review} and references therein) have extended KRR to involve task correlations via matrix-valued kernels. However, most of the literature concerns kernels for which the tasks are fixed at training time. An alternative approach, allowing the generalization to new tasks more straightforwardly, is to use the tensor product pairwise kernel \cite{basilico2004unifying,oyama2004using,Benhur2005,park2009pairwise,Hayashi2012,Bonilla2012,pahikkala2013efficient}, in which kernels are defined on object-task pairs
\eqn{\label{pairwisekernel}
\kkernelf(\bm{x},\overline{\bm{x}}) = \kkernelf\left(\left(\bm{d},\bm{t}\right),\left(\overline{\bm{d}},\overline{\bm{t}}\right)\right)=\kernelf\left(\bm{d},\overline{\bm{d}}\right)\gkernelf\left(\bm{t},\overline{\bm{t}}\right)
}
as a product of the data kernel $\kernelf$ and task kernel $\gkernelf$. Given that $\bm{K}\in\mathbb{R}^{\osize\times\osize}$ and $\bm{G}\in\mathbb{R}^{\qsize\times\qsize}$ are the kernel matrices for the data points and tasks, respectively, the kernel matrix for the object-task pairs is, for a complete training set, the tensor product $\kkernelm=\bm{K}\otimes\bm{G}$. If the training set is not complete, the kernel matrix is a principal sub-matrix of $\kkernelm$. Pairwise KRR seeks for a prediction function of type
\eq{
\predfun(\bm{x}) = \sum_{i=1}^{\lsize}
\alpha_i \kkernelf(\bm{x},\bm{x}_i) \,,
}
where $\alpha_i$ are parameters that minimize the following objective function:
\begin{align}\label{tikhonov}
J(\bm{\boldsymbol\alpha})=(\bm{\kkernelm}\bm{\boldsymbol\alpha}-\bm{y})\transpose(\bm{\kkernelm}\bm{\boldsymbol\alpha}-\bm{y})+\regparam\bm{\boldsymbol\alpha}\transpose\kkernelm\bm{\boldsymbol\alpha},
\end{align}
whose minimizer can be found by solving the following system of linear equations:
\eqn{\label{kronsystem}
\left(\kkernelm+\regparam\idmatrix\right)\bm{\boldsymbol\alpha}=\bm{y}.
}
Several authors have pointed out that, while the size of the above system is considerably large, its solution can be found efficiently via tensor algebraic optimization \citep{vanloan2000ubiquitous,martin2006shiftedkron,Kashima2009linkprob,Raymond2010scalable,pahikkala2010conditional,alvarez2012review}. Namely, the complexity scales roughly of order $O(\arrowvert\objset\arrowvert^3+\arrowvert\taskset\arrowvert^3)$ which is required by computing the singular value decomposition (SVD) of both the object and task kernel matrices, but the complexities can be scaled down even further by using sparse kernel matrix approximations.

However, the above computational short-cut only concerns the case in which the training set is complete. If some of the pairs are missing or if there are several occurrences of certain pairs, one has to resort, for example, to gradient descent based training approaches. While these approaches can also be accelerated via tensor algebraic optimization, they still remain considerably slower than the SVD-based approach. A serious short-coming of the approach is that when generalizing to new tasks, the whole training procedure needs to be re-done with the new training set that contains the union of the auxiliary data and the target data. If the amount of auxiliary data is large, as one would hope in order to expect any positive transfer to happen, this makes generalization to new tasks on-the-fly computationally impractical.

\subsection{Two-step kernel ridge regression}

\begin{algorithm}[t]
  \begin{algorithmic}[1]
    \State $\bm{C} \gets \textnormal{argmin}_{\bm{C}\in\mathbb{R}^{\osize\times\qsize}}\left\{\Arrowvert\bm{C}\tkernelm-\bm{Y}\Arrowvert_F^2+\regparam_t\trace(\bm{C}\tkernelm\bm{C}\transpose)\right\}$
    \State $\bm{z}\gets\left(\bm{z}_\mathcal{L}\transpose,(\bm{C}_\mathcal{U}\bm{g})\transpose\right)\transpose$
    \State $\bm{a}\gets\argmin_{\bm{a}\in\mathbb{R}^{\osize}} \left\{ (\dkernelm\bm{a}-\bm{z})\transpose(\dkernelm\bm{a}-\bm{z})+\regparam_d\bm{a}\transpose\dkernelm\bm{a} \right\}$
    \State \textbf{return} $\predfun_\bm{t}(\cdot)=\sum_{i=1}^\osize {a}_i\kernelf(\bm{d}_i,\cdot)$
  \caption{Two-step kernel ridge regression}
  \label{algtwostep}
  \end{algorithmic}
\end{algorithm}

Next, we present a two-step procedure for performing transfer learning. In the following, we assume that we are provided a training set in which every auxiliary task has the same labeled training objects. This assumption is fairly realistic in many practical settings, since one can carry out, for example, a preliminary completion step by using the extensive toolkit of missing value imputation or matrix completion algorithms. A newly given target task, in contrast, is assumed to have only a subset of the training objects labeled. That is, the training set consisting of both the auxiliary and the target tasks is incomplete, because of the missing object labels of the target task, ruling out the direct application of the SVD-based training.
To cope with this incompleteness, we consider an approach of performing the learning in two steps, of which the first step is used for completing the training set for the target tasks part (almost full cold start) and the second step for building a model for the target task (full cold start). A particular benefit of the approach is that the first phase where a model is trained on auxiliary data needs to be performed only once, and the resulting model may be subsequently re-used when new target tasks appear.

Let $\mathcal{L}\subseteq\objset$ and $\mathcal{U}\subseteq\objset$ be the set of objects that are, respectively, labeled and unlabeled for the target task. Moreover, let $\bm{Y}$ now denote the matrix of labels for the auxiliary tasks and $\bm{z}_\mathcal{L}\in\mathbb{R}^{\arrowvert\mathcal{L}\arrowvert}$ the vector of known labels for the target task. Furthermore, let $\bm{g}\in\mathbb{R}^{\qsize}$ denote the vector of task kernel evaluations between the target task and the auxiliary tasks, e.g.
$\bm{g}=\left(\gkernelf(\bm{t},\bm{t}_1),\ldots,\gkernelf(\bm{t},\bm{t}_\qsize)\right)\transpose$, where $\bm{t}$ is the target task and $\bm{t}_i$ the auxiliary tasks.
Finally, let $\regparam_t$ and $\regparam_d$ be the regularization parameters for the first and the second learning steps, respectively.
The two-step approach is summarized in Algorithm~1.
The first training step (line 1) can be carried out by training a multi-label KRR model, in which a matrix $\bm{C}$ of parameters is estimated. The second step (lines 2-4) employs a single-label KRR, in which a vector $\bm{a}$ of parameters is fitted to the data.

\subsection{Computational considerations and model selection}\label{compsection}

\begin{algorithm}[t]
{\small
  \begin{algorithmic}[1]
    \Require $\bm{Y}\in\mathbb{R}^{\osize\times\qsize},\bm{\Phi}\in\mathbb{R}^{\osize\times\odsize},\bm{\Psi}\in\mathbb{R}^{\qsize\times\qdsize},\bm{g}\in\mathbb{R}^{\qsize},\bm{z}_\mathcal{L}\in\mathbb{R}^{\arrowvert\mathcal{L}\arrowvert}$ with $\odsize\leq\osize$ and $\qdsize\leq\qsize$.
    \State $\bm{U},\sqrt{\bm{\Sigma}},\bm{V}\gets\textnormal{SVD}(\bm{\Phi})$, with $\bm{U}\in\mathbb{R}^{\osize\times\odsize}$, $\bm{V}\in\mathbb{R}^{\odsize\times\odsize}$ \Comment{$\mathcal{O}(\qsize\qdsize^2)$} \label{svdline1}
    \State $\bm{P},\sqrt{\bm{S}},\bm{Q}\gets\textnormal{SVD}(\bm{\Psi})$, with $\bm{P}\in\mathbb{R}^{\qsize\times\qdsize},\bm{Q}\in\mathbb{R}^{\qdsize\times\qdsize}$\Comment{$\mathcal{O}(\osize\odsize^2)$} \label{svdline2}
    \State $e\gets\infty$
    \For{$\overline{\regparam_t}\in\{\textnormal{Grid of parameter values}\}$} \label{firstloopstart}
      \For{$j=1,\ldots,\qsize$}
         $\phantom{W}\widetilde{G}_{j,j}\gets\bm{P}_j(\textnormal{diag}((\bm{S}+\overline{\regparam_t}\idmatrix)^{-1})\odot\bm{P}_j\transpose)$\Comment{$\mathcal{O}(\qsize\qdsize)$}
      \EndFor
      \State $\overline{\bm{C}}\gets\bm{Y}\bm{P}(\bm{S}+\overline{\regparam_t}\idmatrix)^{-1}\bm{P}\transpose$\Comment{$\mathcal{O}(\osize\qsize\qdsize)$}
      \For{$i=1,\ldots,\osize$ \textbf{and} $j=1,\ldots,\qsize$} $\overline{R}_{i,j}\gets Y_{i,j} - \left(\widetilde{G}_{j,j}\right)^{-1}\overline{C}_{i,j}$ \Comment{$\mathcal{O}(\osize\qsize)$}
      \EndFor
      \State $\overline{e}\gets\mathcal{E}(\overline{\bm{R}},\bm{Y})$\Comment{Error between labels and LOO predictions}
      \If{$\overline{e}<e$}$\phantom{W}\regparam_t,e,\bm{R},\bm{C}\gets\overline{\regparam_t},\overline{e},\overline{\bm{R}}, \overline{\bm{C}}$
      \EndIf \label{firstloopend}
    \EndFor
    \State $e\gets\infty$
    \For{$\overline{\regparam_d}\in\{\textnormal{Grid of parameter values}\}$} \label{secondloopstart}
      \For{$i=1,\ldots,\osize$}
         $\phantom{W}\widetilde{K}_{i,i}\gets\bm{U}_i(\textnormal{diag}((\bm{\Sigma}+\overline{\regparam_d}\idmatrix)^{-1})\odot\bm{U}_i\transpose)$\Comment{$\mathcal{O}(\osize\odsize)$}
      \EndFor
      \State $\overline{\bm{A}}\gets\bm{U}(\bm{\Sigma}+\overline{\regparam_d}\idmatrix)^{-1}\bm{U}\transpose\bm{R}$\Comment{$\mathcal{O}(\osize\qsize\odsize)$}
      \For{$i=1,\ldots,\osize$ \textbf{and} $j=1,\ldots,\qsize$} $\overline{T}_{i,j}\gets Y_{i,j} - \left(\widetilde{K}_{i,i}\right)^{-1}\overline{A}_{i,j}$ \Comment{$\mathcal{O}(\osize\qsize)$}
      \EndFor
      \State $\overline{e}\gets\mathcal{E}(\overline{\bm{T}},\bm{Y})$\Comment{Error between labels and LOO predictions}
      \If{$\overline{e}<e$}$\phantom{W}\regparam_d,e,\bm{T},\bm{A}\gets\overline{\regparam_d},\overline{e},\overline{\bm{T}}, \overline{\bm{A}}$
      \EndIf \label{secondloopend}
    \EndFor
    \State $\bm{z}\gets\left(\bm{z}_\mathcal{L}\transpose,(\bm{C}_\mathcal{U}\bm{g})\transpose\right)\transpose$
    \State $\bm{a}\gets\bm{U}(\bm{\Sigma}+\regparam_d\idmatrix)^{-1}\bm{U}\transpose\bm{z}$\Comment{$\mathcal{O}(\osize\odsize)$}
    \State \textbf{return} $\predfun_\bm{t}(\cdot)=\sum_{i=1}^\osize\bm{a}_i\kernelf(\bm{d}_i,\cdot)$
  \caption{Two-step with LOOCV-based automatic model selection}
  \label{algloocv}
  \end{algorithmic}}
\end{algorithm}

Let $\odsize$ and $\qdsize$ denote the feature space dimensionalities of the object and task kernels, respectively. These dimensions can be reduced, for example, by the Nystr{\"o}m method in order to lower both the time and space complexities of kernel methods \citep{schoelkopf1999inputspace}, and hence in the following we assume that $\odsize\leq\osize$ and $\qdsize\leq\qsize$. Let $\bm{\Phi}\in\mathbb{R}^{\osize\times\odsize}$ and $\bm{\Psi}\in\mathbb{R}^{\qsize\times\qdsize}$ be the matrices containing the feature representations of the training objects and tasks in $\objset$ and $\taskset$, respectively, so that $\bm{\Phi}\bm{\Phi}\transpose=\bm{K}$ and $\bm{\Psi}\bm{\Psi}\transpose=\bm{G}$. Let $\bm{\Phi}=\bm{U}\sqrt{\bm{\Sigma}}\bm{V}\transpose$ and $\bm{\Psi}=\bm{P}\sqrt{\bm{S}}\bm{Q}\transpose$ be the SVDs of $\bm{\Phi}$ and $\bm{\Psi}$, respectively. Since the ranks of the feature matrices are at most the dimensions of the feature spaces, we can save both space and time by only computing the singular vectors that correspond to the nonzero singular values. That is, we compute the matrices $\bm{U}\in\mathbb{R}^{\osize\times\odsize}$, $\bm{V}\in\mathbb{R}^{\odsize\times\odsize}$, $\bm{P}\in\mathbb{R}^{\qsize\times\qdsize}$, and $\bm{Q}\in\mathbb{R}^{\qdsize\times\qdsize}$ via the economy sized SVD, requiring $\mathcal{O}(\osize\odsize^2+\qsize\qdsize^2)$ time.
The outcomes of the first and second steps of the two-step KRR (e.g. the first and third lines of Algorithm~1) can be, respectively, written as $\bm{C}=\bm{Y}\widetilde{\bm{G}}$ and $\bm{a}=\widetilde{\bm{K}}\bm{z}$, where $\widetilde{\bm{G}}=(\bm{G}+\regparam_t\idmatrix)^{-1}=\bm{U}(\bm{\Sigma}+\regparam_d\idmatrix)^{-1}\bm{U}\transpose$ and $\widetilde{\bm{K}}=(\bm{K}+\regparam_d\idmatrix)^{-1}=\bm{U}(\bm{\Sigma}+\regparam_d\idmatrix)^{-1}\bm{U}\transpose$. Given that the above described SVD components are available, the computational complexity is dominated by the multiplication of the eigenvectors with the label matrix, which requires $\mathcal{O}(\osize\qsize\qdsize)$ time if the matrix multiplications are performed in the optimal order.

We next present an automatic model selection and training approach for the two-step KRR that uses leave-one-out cross-validation (LOOCV) for selecting the values of both $\regparam_t$ and $\regparam_d$. This is illustrated in Algorithm~\ref{algloocv}. It is well known that, for KRR, the LOOCV performance can be efficiently computed without training the model from scratch during each CV round (we refer to \citet{rifkin2007notes} for details). Adapting this to the first step of the two-step KRR, the ``leave-column-out'' performance for the $i$th datum on the $j$th task (e.g. a CV in which each of the columns of $\bm{Y}$ are held out at a time to measure the generalization ability to new columns) can be obtained in constant time from
$
Y_{i,j} - \left(\widetilde{G}_{j,j}\right)^{-1}C_{i,j},
$
given that the diagonal entries of $\widetilde{\bm{G}}$ and the dual variables $C_{i,j}$ are computed and stored in memory. Using the SVD components, both $\widetilde{G}_{j,j}$ and $C_{i,j}$ can be computed in $\mathcal{O}(\qdsize)$ time, which enables the efficient selection of the regularization parameter value with LOOCV. If the value is selected from a set of $t$ candidates and LOOCV is computed for all data points and tasks, the overall complexity is $\mathcal{O}(\osize\qsize\qdsize t)$. This is depicted in lines \ref{firstloopstart}-\ref{firstloopend} of Algorithm~\ref{algloocv}, where the overline symbols denote temporary variables used in the search of the optimal candidate value and $\mathcal{E}$ denotes a prediction performance.

By the definition of the two-step KRR, the second step consists of training a model using the predictions made during the first step as training labels, while the aim is to make good predictions of the true labels. Therefore, we select the regularization parameter value for the second step using LOOCV on a multi-label KRR model trained using the LOO prediction matrix $\bm{R}$ obtained from the first step as a label matrix. The second regularization parameter value is thus selected so that the error $\mathcal{E}(\overline{\bm{T}},\bm{Y})$ between the LOO predictions made during the second step and the original labels $\bm{Y}$ is as small as possible. In contrast to the first step, the aim of the second step is to generalize to new data points, and hence the CV is done in the leave-row-out sense, which can again be efficiently computed as $Y_{i,j} - \left(\widetilde{K}_{i,i}\right)^{-1}A_{i,j}$, where $A_{i,j}$ are the model parameters of the multi-label KRR trained row-wise. This is done in lines~\ref{secondloopstart}-\ref{secondloopend} of Algorithm~\ref{algloocv}.
 
The overall computational complexity of the two-step KRR with automatic model selection is $\mathcal{O}(\osize\odsize^2+\qsize\qdsize^2 + \osize\qsize\qdsize t + \osize\qsize\odsize t)$, where the first two terms denote the time required by SVD computations and the two latter the time spent for CV and grid search for the regularization parameter. 
The two-step KRR, in addition to enabling non-zero training sets for the target task, provides a very flexible machinery for CV and model selection. This is in contrast to the ordinary KRR with pairwise tensor product kernels for which such short-cuts are not available to our knowledge, and there is no efficient closed form solution available for the almost full cold start settings. Note also that, while the above described method separately selects the regularization parameter values for the tasks and the data, the method is easy to modify so that it would select a separate regularization parameter value for each task and for each datum (e.g. altogether $\osize+\qsize$ parameters), thus allowing considerably more degrees of freedom. However, the consideration of this variation is omitted due to the lack of space.

\section{Theoretical considerations}

Here, we analyze the two-step learning approach by studying its connections to learning with pairwise tensor product kernels of type (\ref{pairwisekernel}). These two approaches coincide in an interesting way for full cold start problems, a special case in which the target task has no labeled data at all. This, in turn, allows us to show the consistency of the two-step KRR via its universal approximation and spectral regularization properties.

The connection between the two-step and pairwise KRR is characterized by the following result. 
\begin{proposition}
Let us consider a full cold start setting with a complete training set. Let $\predfun_\bm{t}(\cdot)$ be a model trained with two-step KRR for the target task $\bm{t}$ and $\predfun(\cdot, \cdot)$ be a model trained with an ordinary least-squares regression on the object-task pairs with the following pairwise kernel function on $\objspace\times\taskspace$:
\eqn{\label{twostepkernel}
\kkernelf\left(\left(\bm{d},\bm{t}),(\overline{\bm{d}},\overline{\bm{t}}\right)\right)%\\
=\left(\kernelf\left(\bm{d},\overline{\bm{d}}\right)
+\regparam_d\delta\left(\bm{d},\overline{\bm{d}}\right)\right)
\left(\gkernelf\left(\bm{t},\overline{\bm{t}})
+\regparam_t\delta\left(\bm{t},\overline{\bm{t}}\right)\right)\right)
}
where $\delta$ is the delta kernel whose value is 1 if the arguments are equal and 0 otherwise.
Then, $\predfun_\bm{t}(\bm{d})=\predfun(\bm{t}, \bm{d})$ for any $\bm{d}\in\objspace$.
\end{proposition}
\begin{proof}
 Writing the steps of the algorithm together and denoting $\widetilde{\bm{G}}=\left(\bm{G}+\regparam\idmatrix\right)^{-1}$ and $\widetilde{\bm{K}}=\left(\bm{K}+\regparam\idmatrix\right)^{-1}$, we observe that the model parameters $\bm{a}$ of the target task can also be obtained from the following closed form:
\eqn{
\bm{a}&=\widetilde{\bm{K}}
\bm{Y}\widetilde{\bm{G}}\bm{g}
\,.
}
The prediction for a datum $\bm{d}$ is $\predfun_\bm{t}(\bm{d})=\bm{k}\transpose\bm{a}$, where $\bm{k}\in\mathbb{R}^\osize$ is the vector containing all kernel evaluations between $\bm{d}$ and the training data points.

The kernel matrix corresponding to the complete training set of auxiliary tasks can be expressed as the following tensor product:
$
\kkernelm = \left(\bm{G}+\regparam_t\idmatrix\right)\otimes\left(\bm{K}+\regparam_d\idmatrix\right)\,.
$
The regression problem being
\[
\bm{\boldsymbol\alpha}=\argmin_{\bm{\boldsymbol\alpha}\in\mathbb{R}^{\osize\qsize}}\left\{\left(\ve(\bm{Y})-\kkernelm\bm{\boldsymbol\alpha}\right)\transpose\left(\ve(\bm{Y})-\kkernelm\bm{\boldsymbol\alpha}\right)\right\}\,,
\]
its minimizer can be expressed as
\eqn{
\bm{\boldsymbol\alpha}&=\kkernelm^{-1}\ve(\bm{Y})%\\
=\left(\left(\bm{G}+\regparam_t\idmatrix\right)^{-1}\otimes\left(\bm{K}+\regparam_d\idmatrix\right)^{-1}\right)\ve(\bm{Y})\nonumber\\
&=\ve\left(\left(\bm{K}+\regparam_d\idmatrix\right)^{-1}\bm{Y}\left(\bm{G}+\regparam_t\idmatrix\right)^{-1}\right)%\\
=\ve\left(\widetilde{\bm{K}}\bm{Y}\widetilde{\bm{G}}\right)\;. \label{twostepsol}
}
The prediction for the datum $\bm{d}$ is $(\bm{g}\otimes\bm{k})\transpose\ve\left(\widetilde{\bm{K}}\bm{Y}\widetilde{\bm{G}}\right)=\bm{k}\transpose\widetilde{\bm{K}}\bm{Y}\widetilde{\bm{G}}\bm{g}$.
\qed
\end{proof}

The kernel point of view allows us to consider the universal approximation properties of the learned knowledge transfer models. Recall the concept of universal kernel functions:
\begin{definition} \label{kerneluniversalitydef}
{\bf \citet{Steinwart2002consistency}} A continuous kernel $\kernelf$ on a compact metric space $\anyspace$ (i.e. $\anyspace$ is closed and bounded) is called universal if the reproducing kernel Hilbert space (RKHS) induced by $\kernelf$ is dense in $C(\anyspace)$, where $C(\anyspace)$ is the space of all continuous functions $\predfun : \anyspace \rightarrow \mathbb{R}$.
\end{definition}
The universality property indicates that the hypothesis space induced by an universal kernel can approximate any continuous function to be learned arbitrarily well, given that the available set of training data is large and representative enough, and the learning algorithm can efficiently find the approximation \citet{Steinwart2002consistency}.

\begin{proposition}
The kernel $\kkernelf$ on $\objspace\times\taskspace$ defined in (\ref{twostepkernel}) is universal if the kernels $\kernelf$ on $\objspace$ and $\gkernelf$ on $\taskspace$ are both universal.
\end{proposition}
\begin{proof}
We provide here a high-level sketch of the proof. The details are omitted due to lack of space but they can be easily verified from the existing literature.
The RKHS of sums of reproducing kernels was characterized by \citet{aronszajn1950} as follows:
Let $\hypspace(\kernelf_1)$ and $\hypspace(\kernelf_2)$ be RKHS’s over $\anyspace$ with reproducing kernels $\kernelf_1$ and $\kernelf_2$,
respectively. If $\kernelf = \kernelf_1 + \kernelf_2$ and
$\hypspace(\kernelf)$ denotes the corresponding RKHS,
then
$\hypspace(\kernelf) = \left\{\predfun_1 + \predfun_2 : \predfun_i \in \hypspace(\kernelf_i) , i = 1, 2\right\}$.
Thus, if the object kernel is universal, the sum of the object and delta kernels is also universal and the same concerns the task kernel. The product of two universal kernels is also universal, as considered in our previous work \citep{waegeman2012learninggraded}.\qed
\end{proof}

The full cold start setting with complete auxiliary training set allows us to consider the two-step approach from the spectral filtering regularization point of view \citep{gerfo2008spectral}, an approach that has recently gained some attention due to its ability to study various types of regularization approaches under the same framework. Continuing from (\ref{kronsystem}), we observe that
\[
\bm{\boldsymbol\alpha}=\filterfun_\regparam(\kkernelm)\ve(\bm{Y})=\bm{W}\filterfun_\regparam(\bm{\Lambda})\bm{W}\transpose\ve(\bm{Y}),
\]
where $\kkernelm = \bm{W}\bm{\Lambda}\bm{W}\transpose$ is the eigen decomposition of the kernel matrix $\kkernelm$ and $\filterfun_\regparam$ is a filter function, parameterized by $\regparam$, such that if $\bm{v}$ is an eigenvector of $\kkernelm$ and $\sigma$ is its corresponding eigenvalue, then $\kkernelm\bm{v}=\filterfun_\regparam(\sigma)\bm{v}$.
The filter function corresponding to the Tikhonov regularization being
$
\filterfun_\regparam(\bm{\sigma})=\frac{1}{\bm{\sigma}+\regparam}\,,
$
and the ordinary least-squares approach corresponding to the $\regparam=0$ case, several other learning approaches, such as spectral cut-off and gradient descent, can also be expressed as filter functions, but which cannot be expressed as a penalized empirical error minimization problem analogous to (\ref{tikhonov}).

The eigenvalues of the kernel matrix obtained with the tensor product kernel on a complete training set can be expressed as the tensor product $\bm{\Lambda}=\bm{\Sigma}\otimes\bm{S}$ of the eigenvalues $\bm{\Sigma}$ and $\bm{S}$ of the object and task kernel matrices.
Now, instead of considering the two-step learning approach from the kernel point of view, one can also cast it into the spectral filtering regularization framework, resulting to the following filter function:
\eqn{\label{twostepfilter}
\filterfun_\regparam(\sigma)=\frac{1}{(\sigma_1+\regparam_t)(\sigma_2+\regparam_d)}
=\frac{1}{\sigma_1\sigma_2+\regparam_d\sigma_1+\regparam_t\sigma_2+\regparam_t\regparam_d}\,,
}
where $\sigma_1,\sigma_2$ are the factors of $\sigma$, namely eigenvalues of $\bm{K}$ and $\bm{G}$. This differs from the Tikhonov regularization only by the two middle terms in the denominator if one sets $\regparam=\regparam_t\regparam_d$. In the experiments, we observe that this difference is rather small also in practical cases, making the two-step learning approach a viable alternative for pairwise KRR with ordinary tensor product kernels.

In the following, we assume that the kernel is bounded, that is, there exists $\kappa>0$ such that $\sup_{\bm{x}\in\ispace}\sqrt{\kkernelf(\bm{x},\bm{x})}\leq\kappa$, indicating that the eigenvalues of kernel matrices are in $[0,\kappa^2]$. To further analyze the above filter functions, we follow \citet{Bauer2007regularization,gerfo2008spectral,baldassarre2012multioutput} and say that a function $\filterfun_\regparam:[0,\kappa^2]\rightarrow \mathbb{R}, 0 < \regparam \leq \kappa^2$, parameterized by $0<\regparam\leq\kappa^2$, is an admissible regularizer if there exists constants $D, B,\gamma\in\mathbb{R}$ and $\bar{\nu},\gamma_\nu > 0$ such that
\[
\sup_{0<\sigma\leq\kappa^2}\arrowvert\sigma\filterfun_\regparam(\sigma)\arrowvert\leq D\textnormal{, }
\sup_{0<\sigma\leq\kappa^2}\arrowvert\filterfun_\regparam(\sigma)\arrowvert\leq\frac{B}{\regparam}\textnormal{, }
\sup_{0<\sigma\leq\kappa^2}\arrowvert 1-\sigma\filterfun_\regparam(\sigma)\arrowvert\leq \gamma\;,
\]
\[
\textnormal{and }\sup_{0<\sigma\leq\kappa^2}\arrowvert 1-\sigma\filterfun_\regparam(\sigma)\arrowvert\sigma^\nu\leq \gamma_\nu\regparam^\nu,\forall\nu\in(0,\bar{\nu}].
\]
The admissibility, in turn, ensures that
\eqn{\label{consistency}
R(\hat{\predfun}^\regparam)-\textnormal{inf}_{\predfun\in\hypspace}R(\predfun)
 = \mathcal{O}\left(\lsize^{-\frac{\bar{\nu}}{2\bar{\nu}+1}}\right)
}
holds with high probability, where $R$ denotes the expected prediction error with respect to some unknown probability measure $\rho(\bm{x}, y)$ on the joint space $\ispace\times\mathbb{R}$ of inputs and labels that is,
$
R(\predfun) = \int_{\ispace\times\mathbb{R}}(\predfun(\bm{x})- y)^2 d\rho(\bm{x}, y)\;.
$
We refer to \citet{Bauer2007regularization,gerfo2008spectral,baldassarre2012multioutput} for a detailed consideration and further results.
It is straightforward to see that, analogously to the Tikhonov regularization, the admissibility of the function (\ref{twostepfilter}) is confirmed by $D, B,\gamma, \gamma_\nu,\bar{\nu}=1$ for arbitrary factorizations of $\regparam=\regparam_t\regparam_d$ and $\sigma=\sigma_1\sigma_2$ such that $\regparam_t,\regparam_d>0$ and $\sigma_1,\sigma_2\geq 0$. Thus, function (\ref{twostepfilter}) can be considered under the spectral filtering regularization framework with separate regularization parameter values for objects and tasks.
The universality of the kernel ensures that $\textnormal{inf}_{\predfun\in\hypspace}R(\predfun)$ in (\ref{consistency}) is the error of the underlying regression function to be learned, and the admissibility of the regularizer ensures that $R(\hat{\predfun}^\regparam)$ converges to it when the size of the training set approaches infinity, guaranteeing the consistency of the two-step KRR method.

\section{Experiments}

In the experiments, we compare different types of transfer learning settings in solving three dyadic prediction problems: drug-target, document similarity and protein similarity prediction. We simulate the full and almost full cold start problem as follows. In each experiment, one drug, document or protein is considered to be the target task in question, where the task is to predict the interactions of drugs or similarities of documents or proteins with respect to the target. Further, other tasks formed in the same way are provided as auxiliary information, leading to a full cold start or almost full cold start setting. The experiments are performed 100 times with different training/test set splits, the performances are averages over all repetitions and over all target tasks. The performance is measured using the concordance index \citep{gonen2005concordance} (C-index), also known as the pairwise ranking accuracy
$\frac{1}{|\{(i,j)| y_i > y_j \}|}\sum_{y_i > y_j}H(\hat{y}_i - \hat{y}_j)$, where $y_i$ denote the true and $\hat{y}_i$ the predicted labels, and $H$ is the Heaviside step function. The regularization parameter selection is performed using LOOCV on the training data. For the two-step approach, we select the first regularization parameter via LOOCV on the auxiliary tasks, and the second one via LOOCV on the target task data augmented with predictions from the first step. The implementation of the algorithms used in the experiments will be made available in the RLScore open source machine learning library\footnote{ Available at \url{https://github.com/aatapa/RLScore}}.

The drug-target interaction prediction data\footnote{\url{http://users.utu.fi/aatapa/data/DrugTarget}} \citep{davis2011comprehensive,pahikkala2014realistic} consists of 68 drug compounds and 442 protein targets. The kernel between the drugs is based on the 3D Tanimoto coefficient similarity, and the sequence similarity between the protein targets was computed using the normalized version of the Smith-Waterman score. Further, for each drug-protein pair we have a real-valued label, negative logarithm of the kinase disassociation constant $K_d$, that characterizes the  interaction affinity between the drug and target in question. In each experiment, the task of interest corresponds to one of the drugs in the data set. The goal is to learn to predict for the given drug the $K_d$ values for proteins unseen during the training phase. The performances are always computed over a testing set of 192 protein targets for a given task, i.e. we assess whether for a given target we can discriminate between proteins with more or less affinity for this drug. 

For each task, we vary the number of available training proteins, from 5 to 250. In addition, we have available the training data for the 250 training proteins for the 67 auxiliary tasks. As summarized in Figure~\ref{ref:relationMatrices}, we evaluate a number of different approaches:
\begin{figure}[t]
{\begin{center}
\includegraphics[width=0.3\linewidth]{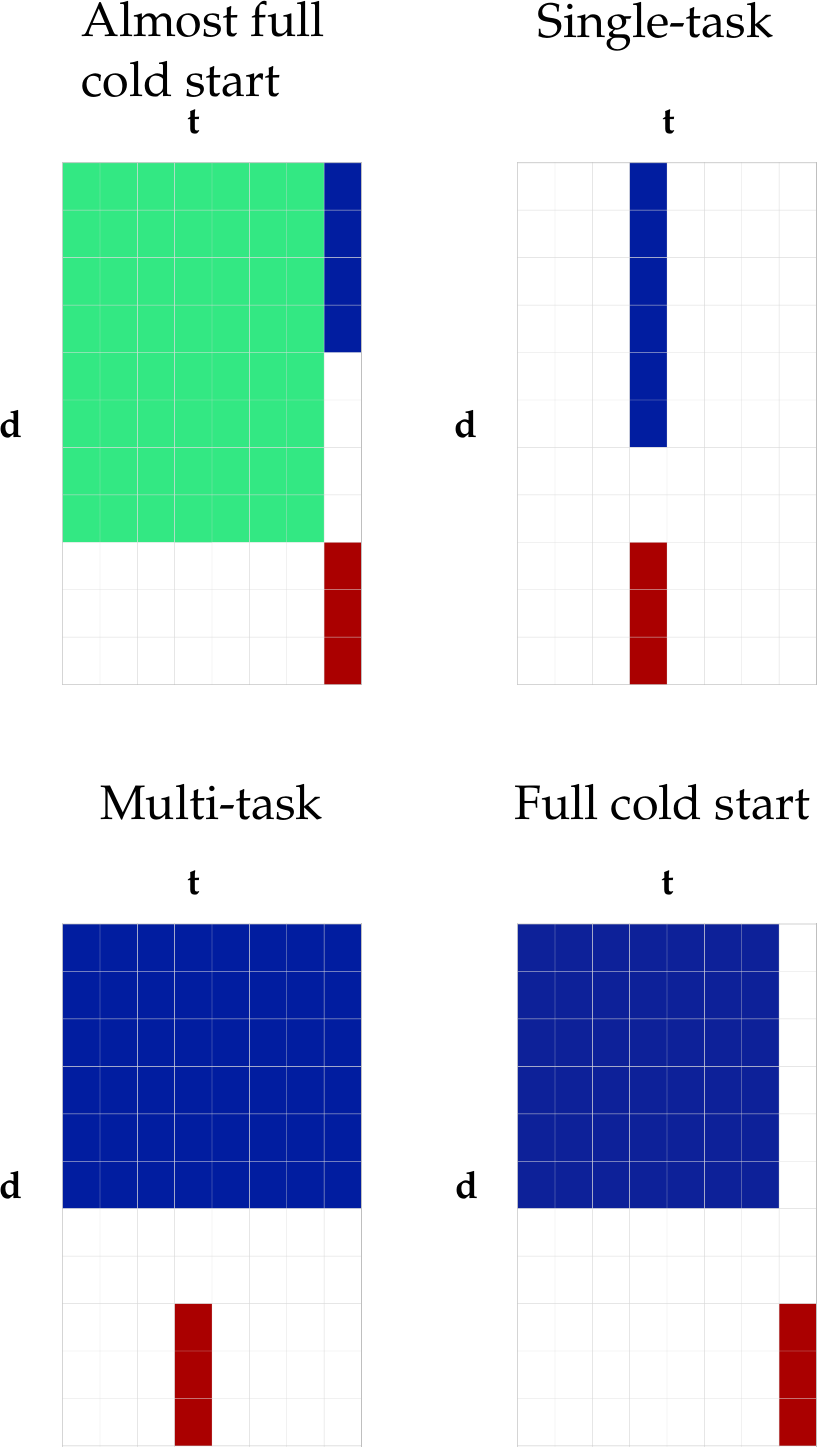}
\end{center}
}
\caption{Overview of the approaches investigated in this article. Green = training data of which the size is constant in the experiments. Blue = training data of which the number of objects varies over different experiments. Red = test data. See text for details.}
\label{ref:relationMatrices}
\end{figure}
\begin{itemize}\setlength{\itemsep}{%
        -0.1\baselineskip plus 1pt minus 1pt%
      }%
\item Single-task: use only training data from target task (traditional regression setting, tackled with KRR)
\item Multi-task: both target task and auxiliary tasks have same amount of training data available (multi-output learning leveraging task correlations, tackled with pairwise tensor product KRR) 
\item Full cold start: no data available for the target task (tackled with pairwise tensor product KRR and two-step KRR)
\item Almost full cold start: use a varying amount of data from the target task, and all the available data from auxiliary tasks (tackled with two-step KRR)
\end{itemize}
We do not consider the  pairwise KRR in the almost full cold start experiment due to computational considerations, as unlike for the two-step approach no closed-form solution exists for the method in this setting, and the iterative conjugate gradient based method has rather poor scalability.

\begin{figure}[t]
        \centering
   \includegraphics[width=0.3\textwidth]{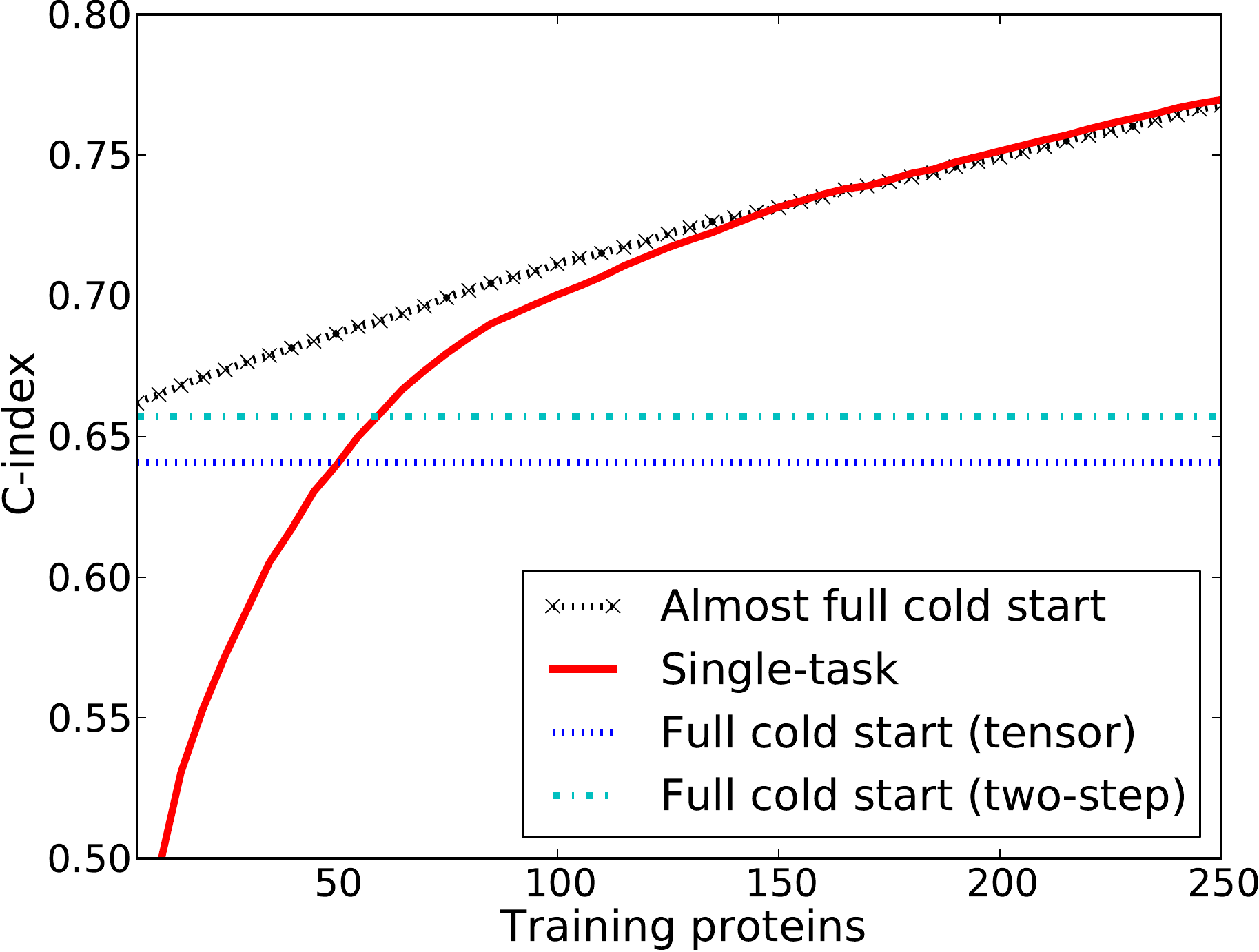}
   \includegraphics[width=0.3\textwidth]{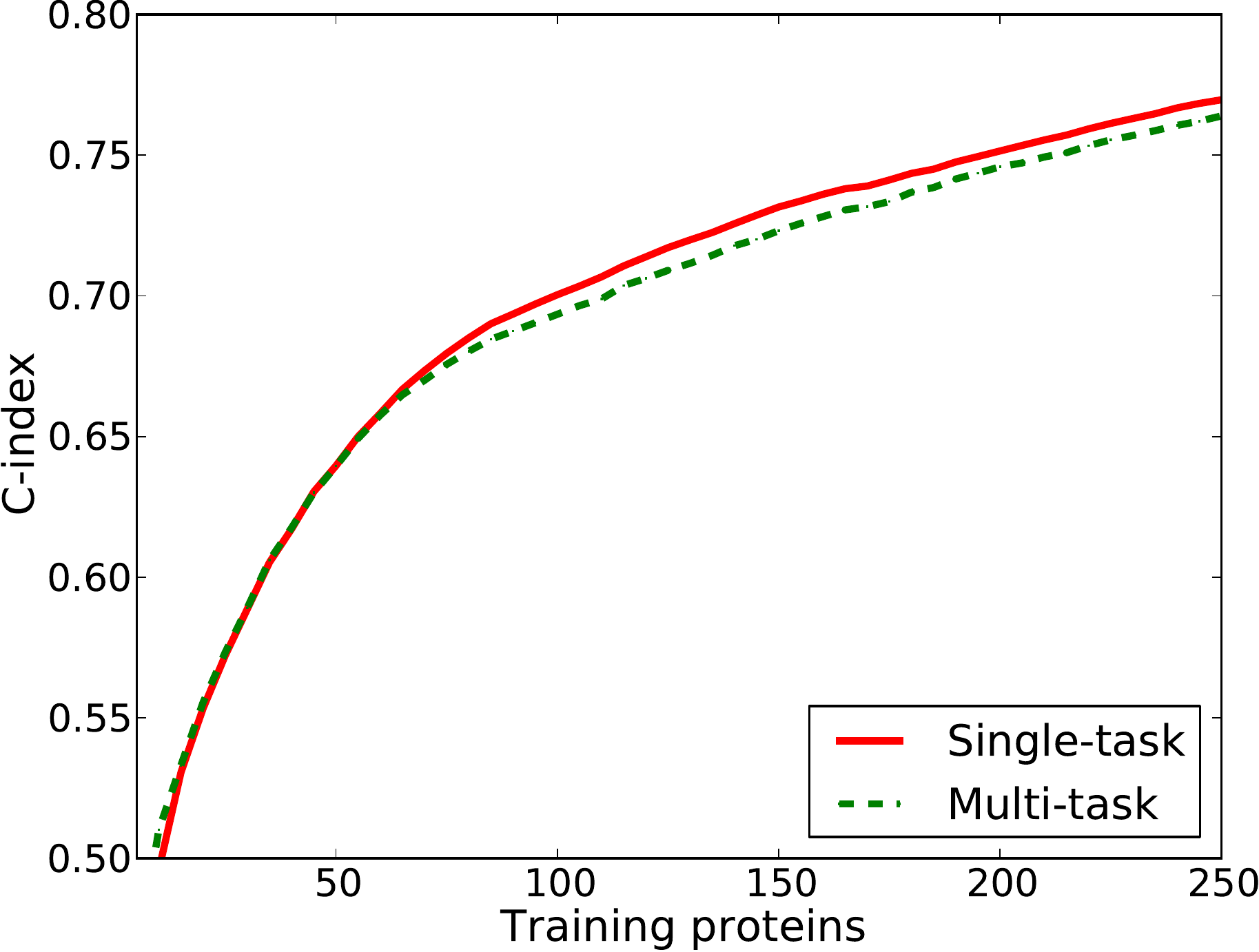}
   \includegraphics[width=0.3\textwidth]{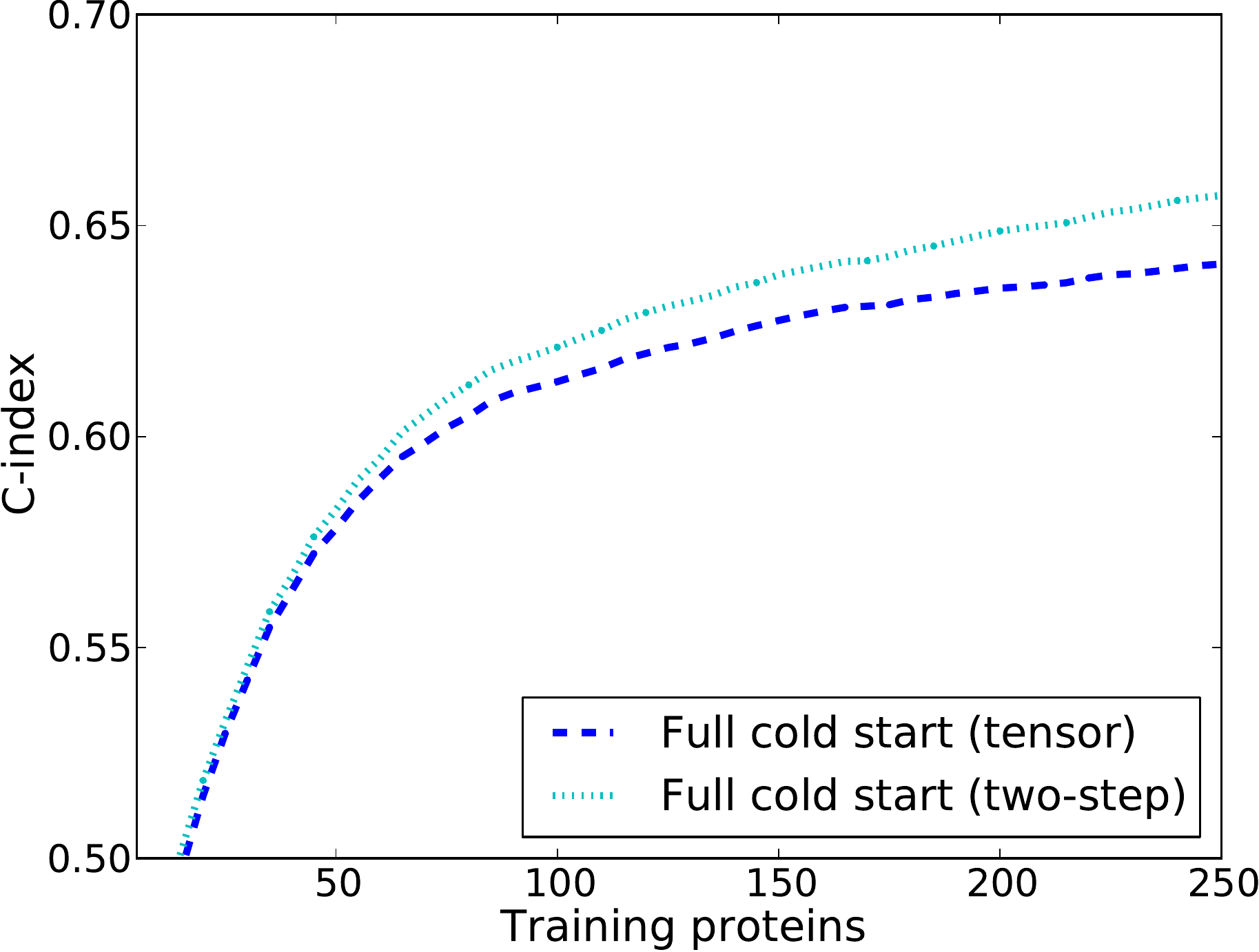}
       
\caption{Learning curves for the drug-target data. Left: target data increased, Middle: target and auxiliary data increased, Right: auxiliary data increased.}
\label{fig:drugtarget}
\end{figure}

In Figure~\ref{fig:drugtarget}, we present the results for the drug-target experiments. In Figure~\ref{fig:drugtarget}~(a) we present an experiment, where all the 67 auxiliary tasks have available the data for all 250 training proteins, and the amount of data available for the target task is varied. It can be seen that learning is possible even in the full cold start setting, where both two-step KRR and pairwise KRR perform much better than randomly. The single-task approach begins to outperform the full cold start setting after the point when one has access to a bit more than 50 training proteins. Combining these two sources of information leads to the best performance up until 150 training proteins. However, once there is enough data available for the target task, there is no longer any positive transfer from the auxiliary tasks.

In Figure~\ref{fig:drugtarget}~(b) we consider the setting, where there is the same amount of data available for both the auxiliary tasks and the target tasks. This setting corresponds closely to the 
traditional multi-output regression problem, the exception being that only the label for the target task is of interest during testing. Here we can see that the multi-task method that uses the task correlation information fails to outperform the simple single-task approach, suggesting that on this type of data one requires significantly more data in the auxiliary tasks compared to the target tasks in order for it to be helpful for learning.

In Figure~\ref{fig:drugtarget}~(c) we consider the full cold start learning setting, while increasing the amount of data available for the auxiliary tasks. Here we observe that the simple two-step approach slightly outperforms pairwise KRR, possibly due to the property that it allows regularizing the drugs and the targets separately. Both approaches generalize to the unknown target task, though the results are still much worse than when having significant amount of data for the target task.

Further, we compare all the considered learning settings on the 20 Newsgroups data\footnote{\url{http://qwone.com/~jason/20Newsgroups/}}. Here, given any target document, the goal is to predict the similarity of other documents with respect to it. This constitutes a three-level ordinal regression task, where documents from the same newsgroup as the target receive the highest rating, documents from similar newsgroups the second highest, and documents from dissimilar newsgroups the lowest rating. These similarities are assigned according to the taxonomy available at the data set web site. The documents are represented using bag-of-words features together with a linear kernel. In the experiments the number of target domain data ranges from 50 to 1500 documents (transfer learning, single-task, multi-task methods), and the number of auxiliary tasks and data available for each either ranges from 50 to 1500 documents (multi-task, full cold start learning), or stays fixed at 2000 documents (transfer learning).

The results are presented in Figure~\ref{ref:figure4}. For the transfer learning approaches, already the starting point of 50 target domain documents suffices to reach a performance that is as good as the single-task method with at least 1500 documents. The multi-task learning setting does not outperform the single-task setting, and while learning is possible in the full cold start setting, some target task data is still required to reach a high predictive performance. Two-step learning slightly outperforms pairwise KRR.

\begin{figure}
\begin{center}
\includegraphics[width=0.4\linewidth]{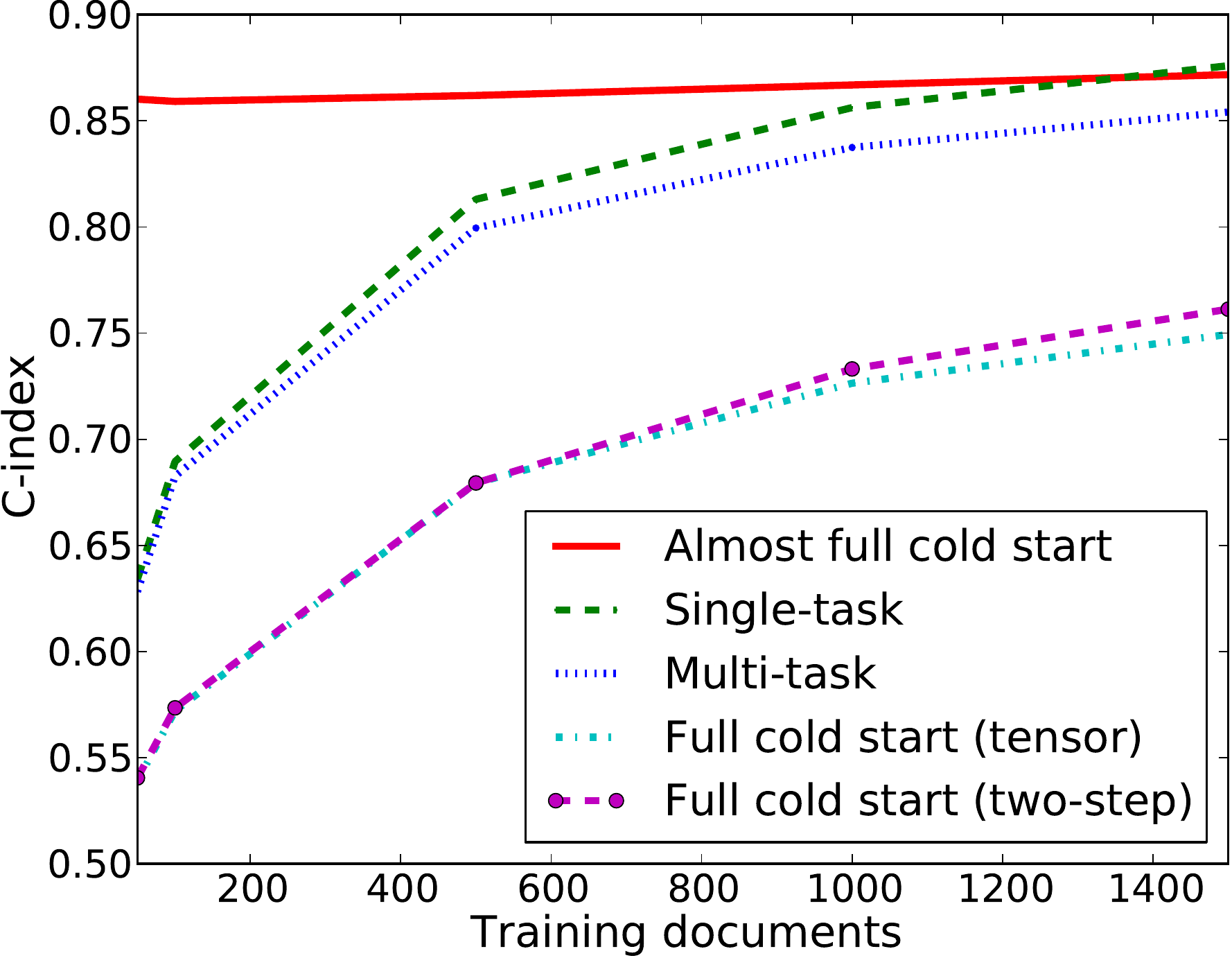}
\includegraphics[width=0.4\linewidth]{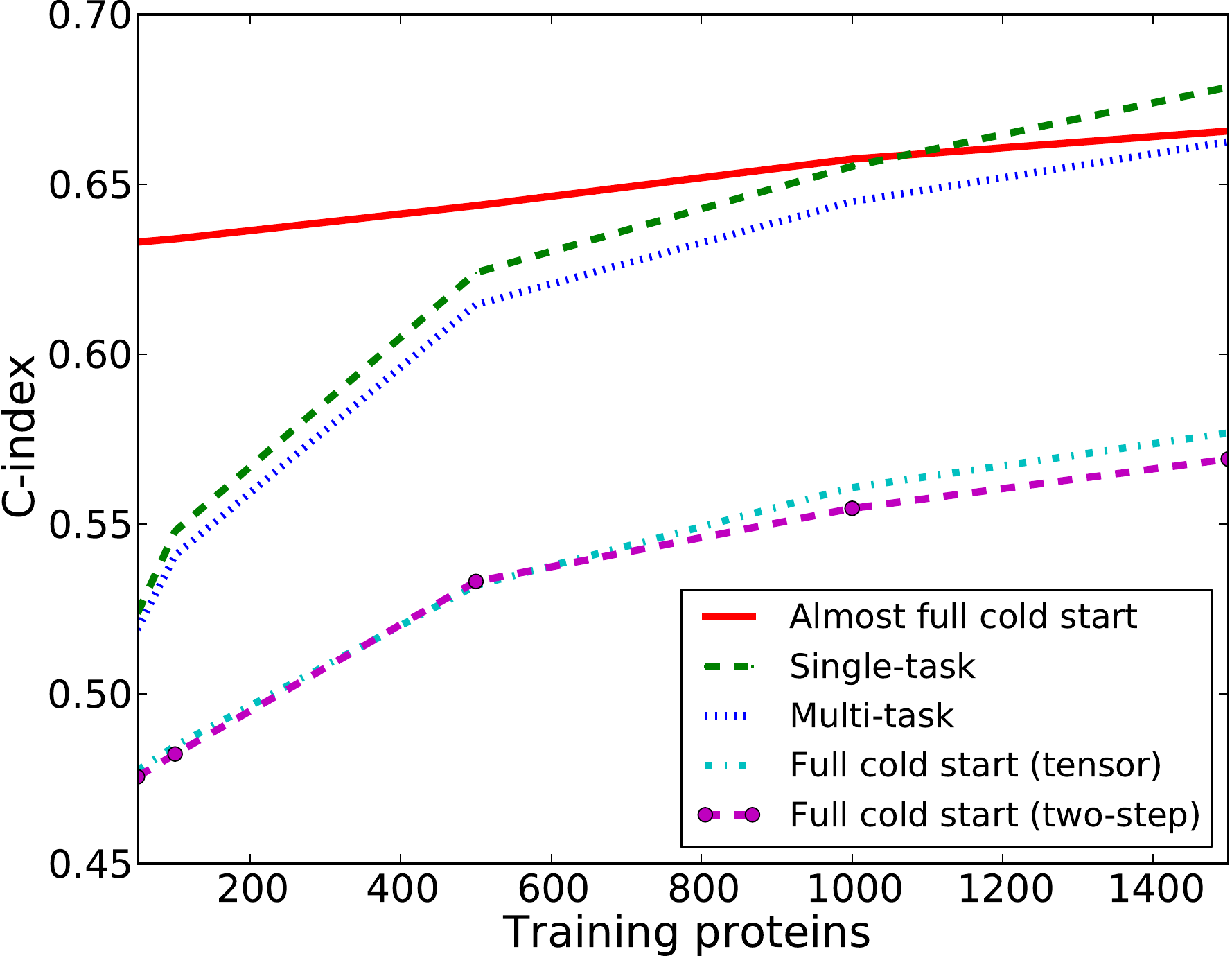}
\end{center}
\caption{Learning curves for the 20 Newsgroups data (left) and the Uniprot data (right).}
\label{ref:figure4}
\end{figure}

The UniProt data was generated by downloading all the protein amino acid sequences with all the gene ontology (GO) annotations of the Universal Protein Resource (UniProt) database. For the amino acid sequences we used the normalized spectrum kernel \citep{leslie2002spectrum}. This kernel is a popular tool for comparing biological sequences without alignments. The normalized spectrum kernel is based on the number of \emph{k}-mers two sequences have in common. In our experiments, \emph{k} was set to three. Two proteins were labeled as 'similar in function' when they had at least one GO term in common. The problem of protein function prediction was thus transformed to a binary classification problem. The experimental setup is the same as for the Newsgroup data, and the results, presented in Figure~\ref{ref:figure4}, are very similar, though at 1500 proteins the performance of the two-step method actually falls below that of the single-task approach.

In all experiments the two-step approach shows itself to be competitive compared to the pairwise learning approach. Previously, \cite{Schrynemackers2013} have in their overview article on dyadic prediction in the biological domain made the observation that in terms of predictive accuracy experimentally there does not seem to be a clear winner between the single-task and multi-task type of learning approaches. Based on our experimental results, a deciding factor on whether one may expect positive transfer from related tasks seems to be based on the amount of data available for the target task. The two-step method performs well in the almost full cold start settings with availability of a significant amount of auxiliary data and only very little data for the target task. But when there is enough data available for the target task, auxiliary data is no longer helpful.

%\section{Conclusion}

%In this paper we analyzed full and almost full cold start settings in dyadic prediction. We proposed a two-step KRR procedure, which is very simple to implement in a computationally efficient manner, yet shows good empirical results on a wide range of applications. Via a theoretical analysis in the spectral filtering framework, we show that the method is closely related to pairwise KRR. 

%\subsubsection*{Acknowledgements}

%Use unnumbered third level headings for the acknowledgements.  All acknowledgements go at the end of the paper.  Be sure to omit any identifying information in the initial double-blind submission!

\bibliographystyle{splncs03}
\bibliography{myBibliography}

\begin{thebibliography}{10}
\providecommand{\url}[1]{\texttt{#1}}
\providecommand{\urlprefix}{URL }

\bibitem{Adams2010}
Adams, R.P., Dahl, G.E., Murray, I.: Incorporating side information into
  probabilistic matrix factorization using {G}aussian processes. Proceedings of
  the 26th Conference on Uncertainty in Artificial Intelligence, pp. 1--9
  (2010)

\bibitem{alvarez2012review}
{\'A}lvarez, M., Rosasco, L., Lawrence, N.D.: Kernels for vector-valued
  functions: a review. Foundation and Trends in Machine Learning  4(3),
  195--266 (2012)

\bibitem{aronszajn1950}
Aronszajn, N.: Theory of reproducing kernels. Transactions of the American
  Mathematical Society  68 (1950)

\bibitem{baldassarre2012multioutput}
Baldassarre, L., Rosasco, L., Barla, A., Verri, A.: Multi-output learning via
  spectral filtering. Machine Learning  87(3),  259--301 (2012)

\bibitem{basilico2004unifying}
Basilico, J., Hofmann, T.: Unifying collaborative and content-based filtering.
  Proceedings of the twenty-first international conference on Machine learning
  (ICML'04), ACM International Conference Proceeding Series, vol.~69 (2004)

\bibitem{Bauer2007regularization}
Bauer, F., Pereverzev, S., Rosasco, L.: On regularization algorithms in
  learning theory. Journal of Complexity  23(1),  52--72 (2007)

\bibitem{Benhur2005}
{Ben-Hur}, A., Noble, W.: Kernel methods for predicting protein-protein
  interactions. Bioinformatics  21 Suppl 1,  38--46 (2005)

\bibitem{Bonilla2012}
Bonilla, E.V., Agakov, F., Williams, C.: Kernel multi-task learning using
  task-specific features. In Proceedings of the Eleventh International
  Conference on Artificial Intelligence and Statistics AISTATS'07 (2007)

\bibitem{davis2011comprehensive}
Davis, M.I., Hunt, J.P., Herrgard, S., Ciceri, P., Wodicka, L.M., Pallares, G.,
  Hocker, M., Treiber, D.K., Zarrinkar, P.P.: Comprehensive analysis of kinase
  inhibitor selectivity. Nature biotechnology  29(11),  1046--1051 (2011)

\bibitem{Fang2011}
Fang, Y., Si, L.: Matrix co-factorization for recommendation with rich side
  information and implicit feedback. Proceedings of the 2nd International
  Workshop on Information Heterogeneity and Fusion in Recommender Systems, pp.
  65--69. HetRec '11 (2011)

\bibitem{gonen2005concordance}
G\"{o}nen, M., Heller, G.: Concordance probability and discriminatory power in
  proportional hazards regression. Biometrika  92(4),  965--970 (2005)

\bibitem{Hayashi2012}
Hayashi, K., Takenouchi, T., Tomioka, R., Kashima, H.: Self-measuring
  similarity for multi-task gaussian process. ICML Workshop on Unsupervised and
  Transfer Learning, JMLR Proceedings, vol.~27, pp. 145--154 (2012)

\bibitem{Jacob2008}
Jacob, L., Vert, J.: Protein-ligand interaction prediction: an improved
  chemogenomics approach", bioinformatics, 24(19):2149-2156, 2008.
  Bioinformatics  241,  2149--2156 (2008)

\bibitem{Kashima2009linkprob}
Kashima, H., Kato, T., Yamanishi, Y., Sugiyama, M., Tsuda, K.: Link
  propagation: A fast semi-supervised learning algorithm for link prediction.
  Proceedings of the SIAM International Conference on Data Mining (SDM 2009),
  pp. 1099--1110 (2009)

\bibitem{Koren2009}
Koren, Y., Bell, R., Volinsky, C.: Matrix factorization techniques for
  recommender systems. Computer  42(8),  30--37 (Aug 2009)

\bibitem{Larochelle2008zerodata}
Larochelle, H., Erhan, D., Bengio, Y.: Zero-data learning of new tasks.
  Proceedings of the 23rd national conference on Artificial intelligence
  (AAAI'08), pp. 646--651 (2008)

\bibitem{leslie2002spectrum}
Leslie, C., Eskin, E., Noble, W.S.S.: {The spectrum kernel: a string kernel for
  SVM protein classification.} Proceedings of the Pacific Symposium on
  Biocomputing, pp. 564--575 (2002)

\bibitem{gerfo2008spectral}
Lo~Gerfo, L., Rosasco, L., Odone, F., De~Vito, E., Verri, A.: Spectral
  algorithms for supervised learning. Neural Computation  20(7),  1873--1897
  (2008)

\bibitem{martin2006shiftedkron}
Martin, C.D., Van~Loan, C.F.: Shifted {Kronecker} product systems. SIAM Journal
  on Matrix Analysis and Applications  29(1),  184--198 (2006)

\bibitem{menon2010loglinear}
Menon, A., Elkan, C.: A log-linear model with latent features for dyadic
  prediction. ICDM, pp. 364--373 (2010)

\bibitem{oyama2004using}
Oyama, S., Manning, C.: Using feature conjunctions across examples for learning
  pairwise classifiers. Proceedings of the European conference on Machine
  learning and Knowledge Discovery in Databases, Lecture Notes in Computer
  Science, vol. 3201, pp. 322--333 (2004)

\bibitem{pahikkala2014realistic}
Pahikkala, T., Airola, A., Pietil{\"a}, S., Shakyawar, S., Szwajda, A., Tang,
  J., Aittokallio, T.: Toward more realistic drug-target interaction
  predictions. Briefings in Bioinformatics  (2014), {Accepted for publication.
  DOI: 10.1093/bib/bbu010}

\bibitem{pahikkala2013efficient}
Pahikkala, T., Airola, A., Stock, M., Baets, B.D., Waegeman, W.: Efficient
  regularized least-squares algorithms for conditional ranking on relational
  data. Machine Learning  93(2-3),  321--356 (2013)

\bibitem{pahikkala2010conditional}
Pahikkala, T., Waegeman, W., Airola, A., Salakoski, T., De~Baets, B.:
  Conditional ranking on relational data. Proceedings of the European
  conference on Machine learning and Knowledge Discovery in Databases, Lecture
  Notes in Computer Science, vol. 6322, pp. 499--514 (2010)

\bibitem{pahikkala2010reciprocalkm}
Pahikkala, T., Waegeman, W., Tsivtsivadze, E., Salakoski, T., De~Baets, B.:
  Learning intransitive reciprocal relations with kernel methods. European
  Journal of Operational Research  206(3),  676--685 (November 2010)

\bibitem{pan2010surveytransfer}
Pan, S.J., Yang, Q.: A survey on transfer learning. IEEE Transactions on
  Knowledge and Data Engineering  22(10),  1345--1359 (2010)

\bibitem{park2009pairwise}
Park, S.T., Chu, W.: Pairwise preference regression for cold-start
  recommendation. Proceedings of the Third ACM Conference on Recommender
  Systems, pp. 21--28 (2009)

\bibitem{park2012flaws}
Park, Y., Marcotte, E.M.: Flaws in evaluation schemes for pair-input
  computational predictions. Nature Methods  9(12),  1134--1136 (Dec 2012)

\bibitem{Raymond2010scalable}
Raymond, R., Kashima, H.: Fast and scalable algorithms for semi-supervised link
  prediction on static and dynamic graphs. Proceedings of the European
  conference on Machine learning and Knowledge Discovery in Databases, Lecture
  Notes in Computer Science, vol. 6323, pp. 131--147 (2010)

\bibitem{rifkin2007notes}
Rifkin, R., Lippert, R.: Notes on regularized least squares. Tech. Rep.
  MIT-CSAIL-TR-2007-025, Massachusetts Institute of Technology, Cambridge,
  Massachusetts, USA (2007)

\bibitem{schoelkopf1999inputspace}
Sch{\"o}lkopf, B., Mika, S., Burges, C., Knirsch, P., M{\"u}ller, K.R.,
  R{\"a}tsch, G., Smola, A.: Input space versus feature space in kernel-based
  methods. IEEE Transactions On Neural Networks  10(5),  1000--1017 (1999)

\bibitem{Schrynemackers2013}
Schrynemackers, M., K{\"u}ffner, R., Geurts, P.: {On protocols and measures for
  the validation of supervised methods for the inference of biological
  networks}. Front Genet.  4,  262 (2013)

\bibitem{Schrynemackers2014}
Schrynemackers, M., Wehenkel, L., Babu, M.M., Geurts, P.: {Classifying pairs
  with trees for supervised biological network inference}. Submitted manuscript
   (2014)

\bibitem{Shan2010}
Shan, H., Banerjee, A.: Generalized probabilistic matrix factorizations for
  collaborative filtering. ICDM, pp. 1025--1030 (2010)

\bibitem{Steinwart2002consistency}
Steinwart, I.: On the influence of the kernel on the consistency of support
  vector machines. Journal of Machine Learning Research  2,  67--93 (2002)

\bibitem{vanloan2000ubiquitous}
Van~Loan, C.F.: The ubiquitous kronecker product. Journal of Computational and
  Applied Mathematics  123(1--2),  85--100 (2000)

\bibitem{waegeman2012learninggraded}
Waegeman, W., Pahikkala, T., Airola, A., Salakoski, T., Stock, M., De~Baets,
  B.: A kernel-based framework for learning graded relations from data. IEEE
  Transactions on Fuzzy Systems  20(6),  1090--1101 (2012)

\bibitem{Zhou2012}
Zhou, T., Shan, H., Banerjee, A., Sapiro, G.: Kernelized probabilistic matrix
  factorization: Exploiting graphs and side information. SDM, pp. 403--414
  (2012)

\end{thebibliography}

\end{document}